\documentclass[12pt]{article}
\usepackage{times}

\usepackage[numbers]{natbib}
\usepackage[utf8]{inputenc} 
\usepackage{amsmath}
\usepackage{amssymb}
\usepackage{amsthm}
\usepackage[T1]{fontenc}    
\usepackage{hyperref}       
\usepackage{url}            
\usepackage{booktabs}       
\usepackage{amsfonts}       
\usepackage{nicefrac}       
\usepackage{microtype}      
\usepackage{xcolor}         
\usepackage{fullpage} 
\usepackage{amsmath,amsthm,amssymb}
\usepackage{graphicx}
\usepackage{cleveref}
\usepackage{subcaption}
\usepackage{comment}
\usepackage{bbm}
\usepackage{algorithm}

\usepackage[noend]{algpseudocode}
\usepackage{textcomp}
\usepackage{wrapfig}
\usepackage{float}

\usepackage{pgfplots}
\usepackage{multirow}
\pgfplotsset{compat=1.8}
\tikzset{elegant/.style={smooth,thick,samples=500,magenta}}

\usepackage{enumitem}
\usepackage{comment} 
\usepackage{hyperref}
\usepackage{natbib}
\usepackage{bm}

\usepackage{thmtools}
\usepackage{thm-restate}
\usepackage{times}

\usepackage{amssymb}
\usepackage{amsmath}
\usepackage{amsfonts}
\usepackage{amsthm}

\usepackage[utf8]{inputenc} 
\usepackage[T1]{fontenc}    
\usepackage{url}            
\usepackage{booktabs}       
\usepackage{amsfonts}       
\usepackage{nicefrac}       
\usepackage{microtype}      
\usepackage{xcolor}

\newcommand{\edim}{\mathrm {dim_E}}

\newtheorem{theorem}{Theorem}[section]

\newtheorem{assumption}{Assumption}[section]

\newtheorem{lemma}[theorem]{Lemma}
\newtheorem{corollary}[theorem]{Corollary}

\newtheorem{definition}[theorem]{Definition}

\newtheorem{remark}{Remark}[section]

\def\R{\mathbb{R}}

\def\cA{\mathcal{A}}
\def\cB{\mathcal{B}}

\def\cF{\mathcal{F}}

\def\cH{\mathcal{H}}

\def\cS{\mathcal{S}}

\def\cS{\mathcal{S}}

\def\cX{\mathcal{X}}

\def\approxcorrect{\cmark\kern-1.4ex\raisebox{.30ex}{$\xmark$}}

\newcommand{\idxn}[1][]{\ifthenelse{\equal{#1}{}}{\mathbb{INDQ}_n}{\mathbb{INDQ}_{#1}}}

\newcommand{\beq}{\begin{equation}}

\newcommand{\eeq}{\end{equation}}

\def \endprf{\hfill {\vrule height6pt width6pt depth0pt}\medskip}

\DeclareMathOperator*{\argmax}{arg\,max}

\title{A Short Note on the Relationship of Information Gain and Eluder Dimension}

\author{
	Kaixuan Huang\textsuperscript{1$\ast$}, Sham M. Kakade\textsuperscript{2,3$\ast$},\\ Jason D. Lee\textsuperscript{1$\ast$}, Qi Lei\textsuperscript{1}\thanks{Alphabetical order. Email: \url{kaixuanh@princeton.edu}, \url{sham@cs.washington.edu}, \url{jasonlee@princeton.edu}, \url{qilei@princeton.edu}.}\\
	\\
	\textsuperscript{1}Princeton University \quad \textsuperscript{2}University of Washington\\
	\textsuperscript{3}Microsoft Research
}

\ifdefined\usebigfont

\usepackage{times}
\usepackage[fontsize=13pt]{scrextend}
\AtBeginDocument{
	\newgeometry{left=1.56in,right=1.56in,top=1.71in,bottom=1.77in}
}
\else
\fi

\begin{document}

\maketitle

\renewcommand{\citet}{\cite}

\begin{abstract}
Eluder dimension and information gain are two widely used methods of complexity measures in bandit and reinforcement learning. Eluder dimension was originally proposed as a general complexity measure of function classes, but the common examples of where it is known to be small are function spaces (vector spaces). In these cases, the primary tool to upper bound the eluder dimension is the elliptic potential lemma. Interestingly, the elliptic potential lemma also features prominently in the analysis of linear bandits/reinforcement learning and their nonparametric generalization, the information gain. We show that this is not a coincidence -- eluder dimension and information gain are equivalent in a precise sense for reproducing kernel Hilbert spaces.
\end{abstract}

\section{Introduction}

Eluder dimension is first proposed by \citet{russo2013eluder} to analyze the regret upper bounds of \textit{upper confidence bound} (UCB) algorithms and \textit{Thompson sampling} (TS) algorithms. 
This notion captures the worst-case sample complexity that is required to infer the values of unobserved points using the observed samples, which measures the degree of dependence among a function class.
The regret of UCB algorithms and TS algorithms can be upper bounded by $ \widetilde {\mathcal{O}} (\sqrt{ \edim \cdot \log N \cdot T })$, where $\edim$ and $N$ are the eluder dimension and the covering number of $\cF$, respectively.
Eluder dimensions are also used to establish regret bounds for learning contextual bandits and MDPs with function approximation (See reference therein \citep{osband2014model, du2020agnostic, wang2020reinforcement, foster2020instance}).

Prior literature further investigated some necessary conditions for a small or finite eluder dimension.
In the original work, \citet{russo2013eluder, osband2014model} provided upper bounds of the eluder dimension for (generalized) linear models with bounded norm and finite dimension and for function classes with bounded domain size.
Specifically, all these results have an explicit dependence on the dimension $d$ of the input space or size of the domain.
For other function classes, the understanding of the eluder dimension is limited, and thus bounded eluder dimension is considered as a strong assumption \citep{foster2020beyond}.
Recently, \citet{li2021eluder} analyzed the eluder dimension through the lens of \textit{generalized-rank}, and calculated the eluder dimension of several interesting function classes including ReLU networks.

In these cases, the primary tool to upper bound the eluder dimension is the elliptic potential lemma (e.g. \citep{russo2013eluder,dani2008stochastic,du2021bilinear}). Interestingly, the elliptic potential lemma also features prominently in the analysis of linear bandits/reinforcement learning and their nonparametric generalization, the \textit{information gain}, which guided another line of studies on the complexity measure for bandit or RL \citep{russo2016information,srinivas2009gaussian}:

Information gain is defined as the mutual information between the prior distribution and the noisy observations, which characterizes the reduction of the uncertainty of the underlying function after observing noisy measurements. 
It has an analytic formula when the prior and the noise distributions are Gaussian.
\citet{srinivas2009gaussian} analyzed the GP-UCB algorithm, a Bayesian variant of UCB algorithms for \textit{Gaussian Process} (GP) bandits, and provided a dimension-free regret bounds in terms of maximum information gain. 
Specifically, they established the regret bounds of $\widetilde {\mathcal{O}} (\sqrt{T} (B\sqrt{\gamma_T} + \gamma_T))$, where $\gamma_T$ is the maximum information gain for $T$ samples and $B$ is the upper bound of the RKHS norm of the ground-truth reward function.
Afterwards, several studies improved the dependency on $\gamma_T$ for the regret bounds for GP-UCB algorithms \citep{chowdhury2017kernelized} and extended to GP-TS algorithms \citep{russo2016information}.
\citet{du2021bilinear} provided regret bounds for learning bilinear MDPs in terms of \textit{critical information gain}.

To achieve sublinear regrets, one requires the maximum information gain to grow mildly in the number of samples. For finite-dimensional domains, the maximum information gain only grows logarithmically \citep{srinivas2009gaussian, du2021bilinear}. For infinite-dimensional RKHSs, the growth of the maximum information gain depends on the eigendecay of the reproducing kernels, and specific rates for different kernels can be established \citep{srinivas2009gaussian, vakili2021information}.

With the existing two lines of work, some natural questions arise:
\begin{center} 
{\em Is one complexity measure strictly tighter than the other? \\Or namely, what is the relationship between using the two complexity measures?}
\end{center} 

\paragraph{Contribution.} Therefore in this work, we study the relationship between eluder dimension and information gain, aiming at bridging the two series of work.
Our work shows that when the function class is a ball of an RKHS, the critial information gain and the eluder dimension of an action set are equivalent. \cite{jin2021bellman} showed concurrently one direction of the equivalence that low critical information gain implies low eluder dimension.

\section{Preliminaries} 

\paragraph{Notation.}
Throughout the paper, we will use the following notation:
\begin{itemize}
	\item $\cX$: the action set.
	\item $\cF \subseteq \R^\cX$: the function class.
	\item $\cH$: the \textit{reproducing kernel Hilbert space} (RKHS).\\Its formal definition is deferred in the next paragraph. 
	\item $K: \cX \times \cX \to \R$: the corresponding kernel function \footnote{We do not distinguish between $x \in \cX$ and $K(\cdot, x) \in \cH$. }.
	\item $\cB(S)$: the ball of radius $S$ in $\cH$, i.e., $$\cB(S) = \{f \mid f(x) = \langle \theta_f , x\rangle, \|\theta_f\|\leq S \}.$$ 
\end{itemize}
For an $x \in \cX$, we use $x^\top$ to represent the adjoint (transpose) of $K(\cdot, x) \in \cH$. Therefore, $x x^\top$ is a self-adjoint linear operator from $\cH$ to itself. We now formally introduce the key concepts used in this paper.

\paragraph{Reproducing Kernel Hilbert Space.} 
\begin{definition}[Reproducing Kernel Hilbert Space]
For any set $\cX$, $\cH \subseteq \R^\cX$ is said to be a \textit{reproducing kernel Hilbert space} (RKHS) with respect to the kernel function $K: \cX \times \cX \to \R$, if $\cH$ is a Hilbert space equipped with the inner product $\langle \cdot,\cdot \rangle$. Furthermore, for any $x \in \cX$ we have $K(\cdot,x) \in \cH$ and
\[
    f(x) = \langle f, K(\cdot,x) \rangle, \ \text{ for all } f \in \cH.
\] 
\end{definition}

Recent years have seen a tremendous studies on nonlinear reinforcement learning via eluder dimension and  kernel methods, due to their potential for expressive function approximation. We are interested in analyzing the connection between eluder dimension and information gain when the function class falls into some RKHS.  We then introduce the formal definition of eluder dimension and information gain respectively. 

\paragraph{Eluder Dimension.}

\begin{definition}[$\epsilon$-Dependence]
    An element $x\in \cX$ is $\epsilon$-dependent on a subset $\cS=\{x_1,\dots x_n\} \subseteq \cX$ with respect to the function class $\cF$, if any pair of functions $f,\tilde f\in \cF$ satisfying $\sum_{i=1}^n (f(x_i) -\tilde f(x_i))^2 \leq \epsilon^2$ also satisfies $f(x)-\tilde f(x)\leq \epsilon$. We say $x$ is $\epsilon$-independent of $\cS$ with respect to $\cF$, if $x$ is not $\epsilon$-dependent on $\cS$. 
\end{definition}
The definition of $\epsilon$-dependence generalizes linear dependence (which corresponds to $\epsilon=0$).
In particular, if $x\in \cX$ is $\epsilon$-dependent on the dataset $\cS$, then knowing the values of $f \in \cF$ on $\cS$ allows for a good estimate of $f(x)$ up to $\epsilon$-accuracy.

\begin{definition}[Eluder Dimension]
The $\epsilon$-eluder dimension $\edim(\cF, \cX, \epsilon)$ is the length $d$ of the longest sequence of elements in $\cX$ such that, for some $\epsilon'\geq \epsilon$, every element $x_i$ is $\epsilon'$-independent of its predecessors $(x_1,\dots,x_{i-1})$ with respect to to $\cF$.  
\end{definition}

Notice that the eluder dimension is defined in a sequential manner. The eluder dimension of a set $\cX$ with respect to a function class $\cF$ corresponds to the number of samples needed to adaptively infer information of $\cF$ among $\cX$ in the worst case.

\paragraph{Information Gain.}

\begin{definition}[Information Gain]
Assume that $\cX$ is associated with a kernel $K$ and a corresponding RKHS $\cH$. For any $\lambda>0$, the {information gain} after observing $x_1,\dots,x_T \in \cX$ is defined as 
\[
    \gamma(\lambda; x_1,\dots,x_T) =  \log\det\Big(I+\frac1\lambda \sum_{i=1}^T x_i x_i^\top\Big).
\]
\end{definition}

\begin{remark}
One can show that the information gain can be equivalently written as 
\[
    \gamma(\lambda; x_1,\dots,x_T) = \log\det \Big(I+\frac1\lambda K_T\Big),
\]
where $(K_T)_{i,j}= \langle x_i , x_j \rangle = K(x_i,x_j)$ is the Gram matrix.
\end{remark}

\begin{definition}[Maximum Information Gain]
For any $\lambda>0$, $T\geq1$, the {maximum information gain} for $T$ observations is defined as 
\[
    \gamma_T(\lambda; \cX) =  \max_{x_i \in \cX, 1\leq i \leq T} \gamma(\lambda; x_1,\dots,x_T). 
\]
\end{definition}

The maximum information gain is analogous to the log of the maximum volume of the ellipsoid generated by $T$ points in $\cX$, which captures the geometric structure of $\cX$.

\begin{definition}[Critical Information Gain]
For a fixed constant $c>0$, the {critical information gain} is defined as:
\[
    \tilde \gamma(\lambda, c; \cX) = \min \{ k \mid  \gamma_k(\lambda; \cX) \leq ck\} . 
\]
\end{definition}

The critical information gain is the minimal $T$ that $\gamma_T(\lambda; \cX)$ fails to grow linearly, which can be viewed as the effective dimension of $\cX$ \citep{du2021bilinear}.
\section{Main Results}
With the formal definition of information gain as well as eluder dimension, we now present our main result on connecting them. 
 
Suppose the function class $\cF$ is in the RKHS $\cH$ over the action set $\cX$. We prove that under some mild assumptions, the critical information gain and eluder dimension are equivalent up to constant factors. 
We drop the dependence on $\cX$ and use $\edim(\cF,\epsilon)$ and $\tilde \gamma(\lambda, c)$ to represent the eluder dimension and the critical information gain, respectively.

\begin{theorem}\label{thm:ub}
Assume that $\cF \subseteq \cB(S)$. Set $\lambda=\big(\frac{\epsilon}{2S}\big)^2$, we have
\[
\edim(\cF,\epsilon) <  \tilde \gamma\Big(\lambda, \log \frac{3}{2} \Big). 
\]

\end{theorem}

\begin{theorem}\label{thm:lb}
Assume that $\|x\|_2 \leq B$ for all $x \in \cX$ and $\cF \supseteq \cB(S)$. Then for any $c > \log 2$, $\epsilon < 2BS$, by setting $\lambda=\big(\frac{\epsilon}{2S}\big)^2$, we have
\[ 
    \edim(\cF,\epsilon) >  (\tilde \gamma(\lambda,c) -1 ) \cdot \frac{c -  \log 2 } {\log(1+\frac {B^2}{\lambda})  -\log 2}.
\]
\end{theorem}

Together we have that, when $\cF$ is the ball of radius $S$ in the RKHS $\cH$ over the action set $\cX$, its eluder dimension and critical information gain with proper $\lambda$ are equivalent (up to constant multiplicative  factors).

We defer the complete proof to the appendix and provide the proof sketch as follows.

\subsection{Proof sketchs}

First, we introduce the notion of \textit{increment of information gain} and show that it is closely related to $\epsilon$-dependence.

\begin{definition}[Increment of Information Gain]
For any $\lambda>0$, the increment of information gain of $x \in \cX$ on $\cS = \{x_1,\dots,x_n \} \subseteq \cX$ is defined as
\[
    \Delta \gamma(\lambda; x|\cS) = \gamma(\lambda; x,x_1,\dots,x_n) - \gamma(\lambda; x_1,\dots,x_n).
\]
\end{definition}

The next two lemmas characterize the relationships between large increment of information gain and $\epsilon$-independence. Intuitively, if $x$ is $\epsilon$-independent of $\cS$, then the increment of information gain of $x$ on $\cS$ will be large and vice versa.

\begin{lemma}\label{lem:incre_via_indepenence}
Assume that $\cF \subseteq \cB(S)$. If $x$ is $\epsilon'$-independent of $\cS$ with respect to $\cF$ for some $\epsilon' \geq \epsilon$, then
\[
    \Delta \gamma((\epsilon/2S)^2; x|\cS)  > \log \frac{3}{2}.
\]
\end{lemma}

\begin{lemma}\label{lem:independence_via_incre}
Assume that $\cF \supseteq \cB(S)$. If
\[
    \Delta \gamma((\epsilon/2S)^2; x|\cS)  > \log 2.
\]
then $x$ is $\epsilon$-independent of $\cS$ with respect to $\cF$.

\end{lemma}

By Lemma~\ref{lem:incre_via_indepenence}, a sequence whose elements are $\epsilon$-independent of their predecessors implies a sequence with large increments of information gain. Therefore we can lower bound the critical information gain by eluder dimension. Next we provide the proof of Theorem~\ref{thm:ub}. 
\begin{proof}[proof of Theorem~\ref{thm:ub}]
Assume that $\edim(\cF,\epsilon) = d$ and  $\{x_1,\dots, x_d\}$ is the longest sequence such that there exists an $\epsilon' \geq \epsilon$, for each $i$, $x_i$ is $\epsilon'$-independent of $\{x_1,\dots, x_{i-1}\}$. Then by Lemma~\ref{lem:incre_via_indepenence}, we have for $\lambda=(\epsilon/2S)^2$,
\[
    \Delta \gamma(\lambda; x_i|\{x_1,\dots, x_{i-1}\} )  >  \log \frac{3}{2}.
\]
This means for $k=1,\dots, d$,
\[
    \gamma_k(\lambda;\cX) \geq \gamma(\lambda;x_1,\dots,x_k) > k \log \frac{3}{2}.
\]
Therefore $\edim(\cF,\epsilon) <  \tilde \gamma(\lambda, \log \frac{3}{2};\cX)$,
which is $\min \{k \mid k \log \frac{3}{2} \geq \gamma_k(\lambda;  \cX) \} $.
\end{proof}

Next lemma shows that the increment of information gain can be arranged to be monotonically decreasing, which stems from the submodularity of the log-determinant function.

\begin{lemma}\label{lem:submodularity}
For any $\{ x_1,\dots,x_T \}$ and any $\lambda>0$, we can arrange the order of the $T$ points, such that for $i=1,\dots,T-1$,
\[
    \Delta \gamma(\lambda, x_i|x_{1},\dots,x_{i-1}) \geq \Delta \gamma(\lambda, x_{i+1}|x_{1},\dots,x_{i}).
\]
\end{lemma}

By Lemma~\ref{lem:submodularity}, the first few elements of the maximizers of $\gamma_T(\lambda;\cX)$ will have large increments of information gain, which, by Lemma~\ref{lem:independence_via_incre}, means that each element is $\epsilon$-independent of its predecessors. Therefore we can lower bound the eluder dimension by critical information gain. Next we provide the proof of Theorem~\ref{thm:lb}.

\begin{proof}[proof of Theorem~\ref{thm:lb}]

Let $d = \tilde \gamma(\lambda,c) -1$. Then $ \gamma_d(\lambda) > cd$. Then we can find $\{ x_1,\dots,x_d \} \subseteq \cX$ such that $\gamma(\lambda;x_1,\dots,x_d) > cd$.
Denote $\Delta_i = \Delta \gamma(\lambda;x_i|x_1,\dots,x_{i-1})$. 
By Lemma~\ref{lem:submodularity}, we can arrange those $x_i$'s such that $\Delta_i$'s are monotonically non-increasing.
Let $k = \max \{i \mid   \Delta_i > \log 2 \}$, then by Lemma~\ref{lem:independence_via_incre}, we know that $\{x_1,\dots,x_k\}$ statisfies the definition in the eluder dimension, so $\edim(\cF,\epsilon) \geq k$.

Since $\|x_i\|\leq B $ for all $i$, we have
\[
\gamma_k(\lambda) \leq k\log(1+\frac {B^2}{\lambda}).
\]
Therefore we get
\begin{align*}
cd  &<\gamma(\lambda;x_1,\dots,x_d) \\
&= \sum_{i=1}^k \Delta_i + \sum_{i=k+1}^d \Delta_i  \\
     &\leq \gamma_k(\lambda) + (d-k) \log 2 \\
     &\leq k\log(1+\frac {B^2}{\lambda}) + (d-k) \log 2.
\end{align*}
By the assumption that $\epsilon<2BS$, we have $B^2/\lambda = (2BS)^2/\epsilon^2>1$, then we have
\[
    k  > d \cdot \frac{c -  \log 2 } {\log(1+\frac {B^2}{\lambda})  -\log 2}.
\]
This finishes the proof. 
\end{proof}
\section{Case Study: Comparison of Two Regret Bounds}

In this section, we briefly discuss two existing results on learning episodic MDPs under the completeness assumption with respect to an RKHS \citep{yang2020function, wang2020reinforcement}. These results base on different techniques and provide the regret guarantees in terms of maximum information gain and eluder dimension, respectively.
Using our result on the equivalence of the eluder dimension and the critical information gain, we characterize how the eluder dimension grows in $T$ when the function class is an RKHS-norm ball, and provide a comparison of the two regret bounds in this setting.

\paragraph{Setting.}
We consider learning an episodic MDP $(\cS, \cA, P, H, r, \mu)$ with function approximation. In RL with function approximation, the input domain of the function class $\cF$ is $\cX = \cS \times \cA$, containing all the state-action pairs $x=(s,a)$. We assume that there exists a kernel $K: \cX \times \cX \to \mathbb{R}$ and a corresponding RKHS $\cH$. Therefore, every $x=(s,a)$ can be viewed as an element in $\cH$. We assume that the function class $\cF = \cB(R_Q H)$ is the RKHS norm ball with radius $R_Q H$. We also assume that $\|x\| \leq 1$ for all $x=(s,a) \in \cS \times \cA$.
We make the following completeness assumption on the function class $\cF$ as in \citep{yang2020function, wang2020reinforcement}.
\begin{assumption}[Completeness]
\label{assump:bc}
For any $h \in [H]$ and any $Q$ function $Q:\cS \times \cA \to [0,H]$, we have $\mathbb{T}_h(Q) \in \cF$, where $\mathbb{T}_h$ is the Bellman operator at level $h$, defined as
\[
    \mathbb{T}_h(Q)(s,a) = r_h(s,a) + \sum_{s' \in \cS} P_h(s'|s,a) \max_{a'\in\cA}Q(s',a') \, \, \, \forall \ (s,a) \in \cS \times \cA.
\]
\end{assumption}

First we state the two results from \citep{yang2020function, wang2020reinforcement}. Specifically, Theorem~\ref{thm:yang} provides the regret bound for KOVI algorithm in terms of the maximum information gain, while Theorem~\ref{thm:wang} provides the regret bound for $\cF$-LSVI algorithm in terms of the eluder dimension.


\begin{theorem}[\citet{yang2020function}]\label{thm:yang}
Under Assumption~\ref{assump:bc}, for KOVI, after interacting with the environment for $T$ episodes ($E=TH$ steps), with probability $1-E^{-2}$,
\[
\operatorname{Regret}(T)=\widetilde{\mathcal{O}} \left (H^{2} \cdot \left[\gamma_T(\lambda;\cX)+\max _{h \in[H]} \sqrt{\gamma_T(\lambda;\cX) \cdot \log N_{\infty}^{(ucb)}\left(\epsilon^{*}, h, B_{T}\right)}\right] \cdot \sqrt{T} \right),
\]
where $\lambda = 1 + \frac{1}{T}$, $\epsilon^{*} = H/T$.
\end{theorem}

\begin{theorem}[\citet{wang2020reinforcement}]\label{thm:wang}
Under Assumption~\ref{assump:bc}, after interacting with the environment for $E = TH$ steps, with probability $1-\delta$, $\cF$-LSVI achieves a a regret bound of
\[
    \operatorname{Regret}(T)\leq \sqrt{\iota H^2 E },
\]
where
\[
    \iota \leq C \cdot \log ^{2}(E / \delta) \cdot \edim^{2}\left(\mathcal{F}, \delta / E^{3}\right) \cdot \log \left({N}_\infty\left(\mathcal{F}, \delta / E^{2}\right) / \delta\right) \cdot \log ({N}_\infty(\mathcal{S} \times \mathcal{A}, \delta / E) \cdot E / \delta),
\]
for some constant $C>0$.
\end{theorem}


By our Theorems~\ref{thm:ub} and~\ref{thm:lb}, the regret bound in Theorem~\ref{thm:wang} can be equivalently stated in terms of the critical information gain. In particular, if we set $\delta = E^{-2}$ and notice that $\cF = \cB(R_Q H)$, then we have
\[
    \widetilde {\mathbf {\Omega}} \Big( \tilde \gamma ((4T^{10}H^{12} R_Q^2)^{-1}  , c) \Big ) \leq \edim (\cF, \delta/E^3) \leq \tilde \gamma \Big( (4T^{10}H^{12} R_Q^2)^{-1}  ,\log\frac{3}{2} \Big),
\]
where $\widetilde {\mathbf {\Omega}}$ hides the constant factors and the log factors. Therefore we have the following corollary.
\begin{corollary}[\citet{wang2020reinforcement}]\label{cor:wang}
Under Assumption~\ref{assump:bc}, after interacting with the environment for $E = TH$ steps, with probability $1-E^{-2}$, $\cF$-LSVI achieves a a regret bound of
\[
    \operatorname{Regret}(T)\leq \sqrt{\iota H^2 E },
\]
where
\[
    \iota \leq C \cdot \log ^{2}(E^3) \cdot \tilde \gamma^2 \Big( (4T^{10}H^{12} R_Q^2)^{-1}  ,\log\frac{3}{2} \Big) \cdot \log \left({N}_\infty\left(\mathcal{F}, 1/ E^{4}\right) E^2\right) \cdot \log ({N}_\infty(\mathcal{S} \times \mathcal{A}, 1 / E^3) \cdot E^3),
\]
for some constant $C>0$.
\end{corollary}
Next we focus on the dependency of these regret bounds on $T$, ignoring all the log factors as well as all the other problem-specific constants, such as $H$ and $R_Q$. Both Corollary~\ref{cor:wang} and Theorem~\ref{thm:yang} have an explicit $\sqrt{T}$ factor and a term characterizing how the information gain grows. 
To continue our comparison, we make the additional assumption on the eigenvalues of the kernel operator as in \citep{srinivas2009gaussian, yang2020function}.
\begin{assumption}[Polynomial decay]\label{ass:decay}
Assume that the state-action set $\cX = \cS \times \cA$ is a compact subset of $\mathbb{R}^d$. Furthermore, assume that the kernel function $K$ admits the following orthonormal decomposition
\[
    K(x,x') = \sum_{i=1}^\infty \lambda_i \phi_i(x) \phi_i(x'), 
\]
with $\| \phi_i(x) \|_\infty \leq C_\phi$ and the eigenvalues $\lambda_i$'s of $L_K$ satisfying $\beta$-polynomial decay: $\lambda_i \leq C_p i^{-\beta}$, for some $\beta > 2+1/d$.
\end{assumption}

Under Assumption~\ref{ass:decay}, the regret bound of Theorem~\ref{thm:yang} can be reduced to the following when $d\geq \Omega(1), \beta \geq \Omega(1)$; see \citep{yang2020function} for more details.
\begin{equation}
     \operatorname{Regret}(T)\leq \widetilde {\mathcal{O}} (T^{\frac{d+1}{d+\beta} + 1/2}). \label{eq:yang}
\end{equation}

Next we calculate the dependency on $T$ of Corollary~\ref{cor:wang}. We introduce the following lemma from the appendix of \citep{yang2020function} (Lemmas D.2 and D.6).

\begin{lemma}[\citep{yang2020function}] \label{lem:b}
Under Assumption~\ref{ass:decay}, the maximum information gain can be bounded as
\[
    \gamma_\lambda(T) \leq  \widetilde {\mathcal{O}}( \lambda^{-1/\beta} T^{ (d+1)/(\beta + d)} ).
\]
Therefore, the critical information gain satisfies
\[
    \tilde \gamma(\lambda, c) = \widetilde{\mathcal{O}} ((1/\lambda)^{
    (\beta+d)/[\beta(\beta-1)]}).
\]
Furthermore, the $\log$ of the $L^\infty$ covering number of the RKHS ball of radius $S$ can be bounded as
\[
    \log {N}_\infty \left( \mathcal{B}(S), 1/\epsilon \right) \leq \widetilde {\mathcal{O}}(  (S/\epsilon)^{ 2/(\beta -1 )} ).
\]
\end{lemma}
By Lemma~\ref{lem:b}, the regret bound in Corollary~\ref{cor:wang} \citep{wang2020reinforcement}  is  
\begin{equation}
    \operatorname{Regret}(T)\leq  \widetilde{\mathcal{O}} (T^{\frac{10
    (\beta+d)+ 7\beta}{\beta(\beta-1)} +1/2}). \label{eq:wang}
\end{equation}
By comparing Equation~\eqref{eq:yang} and Equation~\eqref{eq:wang}, we conclude that when the kernel $K$ satisfies polynomial decay (Assumption~\ref{ass:decay}), Equation~\eqref{eq:yang} is better than Equation~\eqref{eq:wang} as long as $\beta \gg d$.

 \section{Conclusion}

In this note, we clarified the relationship between two complexity measures used in bandits and reinforcement learning, the eluder dimension and information gain in RKHS function spaces. Eluder dimension is a general complexity measure defined for any function class; however, information gain as defined in \cite{srinivas2009gaussian} is only defined for RKHS. It is straightforward to extend the information gain in the Bayesian setting when there is a prior over the function class/prior. We conjecture that for function classes that have a natural vector space structure (i.e. convex subset with non-empty interior of a vector space), that the eluder dimension does not yield sublinear regret, unless information gain is sublinear in the sample size $n$. To our knowledge, all commonly encountered function classes with vector space structure do not have bounded eluder dimension, unless they are subsets of an RKHS. We leave to future work to identify complexity measures for sequential decision making that are more structure adaptive, such as sparsity, low-rank, or other measures of low-dimensionality, than eluder dimension and information gain which essentially only capture the extrinsic vector space dimensionality.

\bibliography{ref,simonduref}
\bibliographystyle{alpha}

\newpage
\appendix
\section{Missing proofs of the lemmas}

\paragraph{Additional notation.}

For any positive definite matrix $A$, we define $\|x\|_A^2 = x^\top A x$. 

We first introduce the next lemma.

\begin{lemma}\label{lem:conn}
Assume that $x \in \cX$ and $\cS = \{x_1,\dots, x_n\} \subseteq \cX$. For any $\lambda >0$, denote $V = \lambda I +  \sum_{i=1}^n x_i x_i^\top $, then we have the following holds
\[
    \Delta \gamma(\lambda; x|\cS) = \log (1 + \| x \|_{V^{-1}} ^2).
\]
\end{lemma}

\begin{proof}
By definition we have
\[
\Delta \gamma(\lambda; x|\cS) = \gamma(\lambda; x,x_1,\dots,x_n) - \gamma(\lambda; x_1,\dots,x_n)= \log \det \Big(\frac{1}{\lambda} (V+xx^\top)\Big) - \log \det \Big(\frac{1}{\lambda} V \Big).  
\]
By the Matrix determinant lemma, we have
\begin{align*}
 \det \Big(\frac{1}{\lambda} (V+xx^\top)\Big)  &= \det \Big[\frac{1}{\lambda}V \cdot  (I + V ^{-1} x x^\top)\Big] \\
  &=  \det \Big(\frac{1}{\lambda} V \Big) \cdot \det (I + V ^{-1} x x^\top)] \\
  &=  \det \Big(\frac{1}{\lambda} V \Big) \cdot \det (I + x^\top V ^{-1} x )] \\
  &=  \det \Big(\frac{1}{\lambda} V \Big) \cdot (1 + \| x \|_{V^{-1}} ^2).
\end{align*}
\end{proof}

\paragraph{Proofs of the Lemmas}

\begin{proof}[proof of Lemma~\ref{lem:incre_via_indepenence}]
Let $\cS = \{x_1,\dots,x_n\}$, $\lambda=(\epsilon/2S)^2$, and $V = \lambda I +  \sum_{i=1}^n x_i x_i^\top$.
By definition, if $x$ is $\epsilon'$-independent of ${x_1,\dots, x_n}$ for some $\epsilon' \geq \epsilon$, then there exist $\theta_1 \in \cH$ and $\theta_2 \in \cH$ such that $f_i(x) := \langle \theta_i, x\rangle \in \cF$ statisfies
\begin{equation}
     \langle \theta_1, x \rangle - \langle \theta_2, x \rangle > \epsilon',\label{eq:cond1}
\end{equation}
and 
\begin{equation}
    \|\theta_1\| \leq S,\  \|\theta_2\| \leq S, \    \sum_{i=1}^n  ( \langle \theta_1, x_i \rangle - \langle \theta_2, x_i \rangle )^2 \leq (\epsilon')^2.\label{eq:cond2}
\end{equation}
Then we have
\begin{align*}
    \| \theta_1 - \theta_2 \|_{V}^2 =&  (\theta_1 - \theta_2)^\top \Big (\lambda I +  \sum_{i=1}^n x_i x_i^\top   \Big )  (\theta_1 - \theta_2) \\
  =& (\theta_1 - \theta_2)^\top \Big (\frac{\epsilon^2} { (2S)^2}  I +  \sum_{i=1}^n x_i x_i^\top  \Big  )  (\theta_1 - \theta_2) \\
  = & \frac{\epsilon^2} { (2S)^2} \| \theta_1 - \theta_2 \|^2   +   \sum_{i=1}^n  (\theta_1 - \theta_2) ^\top     x_i x_i^\top      (\theta_1 - \theta_2) \\
  \leq & \epsilon^2 + (\epsilon')^2,
\end{align*}
where the last inequality follows from Equation~\eqref{eq:cond2}.

By Equation~\eqref{eq:cond1}, we have
\begin{align*}
    \epsilon' <& \langle \theta_1 - \theta_2, x \rangle \\
    \leq &   \| \theta_1 - \theta_2 \|_{V} \cdot \|x\|_{V^{-1}}   \tag{Cauchy Inequality}\\
    \leq &   \sqrt{\epsilon^2 + (\epsilon')^2 } \cdot \|x\|_{V^{-1}}.
\end{align*}
Therefore,
\[
    \|x\|_{V^{-1}}^2 > \frac{(\epsilon')^2 }{\epsilon^2 + (\epsilon')^2} \geq \frac{1}{2}.
\]
Finally, by Lemma~\ref{lem:conn}, we have
\[
  \Delta \gamma((\epsilon/2S)^2 ; x|\cS) = \log (1 + \| x \|_{V^{-1}} ^2) > \log\frac{3}{2}.
\]
\end{proof}

\begin{proof}[proof of Lemma~\ref{lem:independence_via_incre}]
Let $\cS = \{x_1,\dots,x_n\}$, $\lambda=(\epsilon/2S)^2$, and $V = \lambda I +  \sum_{i=1}^n x_i x_i^\top$.
By Lemma~\ref{lem:conn}, we know that $\Delta \gamma((\epsilon/2S)^2 ; x|\cS) > \log 2 $ implies $\| x \|_{V^{-1}} > 1$.

Define
\[
    \theta_1 = \frac{\epsilon}{2} \cdot \frac{V^{-1} x} { \sqrt{x^\top V^{-1} x} }
\] 
and $\theta_2 = - \theta_1$. In this case, we have
\begin{align*}
    \|\theta_1\|^2 &= \frac{\epsilon^2}{4} \cdot\frac {(V^{-1/2}  x) ^\top  V^{-1} (V^{-1/2}  x ) }{{x^\top V^{-1} x}  }\\
    &= \frac{\epsilon^2}{4} \cdot\frac {(V^{-1/2}  x) ^\top \Big( \lambda I +  \sum_{i=1}^n x_i x_i^\top   \Big)^{-1} (V^{-1/2}  x ) }{{x^\top V^{-1} x}  }\\ 
    &\leq \frac{\epsilon^2}{4} \cdot\frac {(V^{-1/2}  x) ^\top \Big( \lambda I \Big)^{-1} (V^{-1/2}  x ) }{{x^\top V^{-1} x}  }\\ 
    &= \frac{\epsilon^2}{4} \cdot\frac{(2S)^2}{\epsilon^2} \cdot\frac { (V^{-1/2}  x) ^\top  (V^{-1/2}  x ) }{{x^\top V^{-1} x}  }\\ 
    &=S^2.
\end{align*}
Similarly we have $\|\theta_2\|^2 \leq S^2$.  Therefore $f_i(x) := \langle \theta_i, x\rangle \in \cF$ for $i=1,2$.

Next, we have
\begin{align*}
     \sum_{i=1}^n (\langle \theta_1 - \theta_2 , x_i \rangle ) ^2 =&   \epsilon^2 \cdot   \frac{  (V^{-1} x)^\top \Big (\sum_{i=1}^n x_i x_i^\top \Big )(V^{-1} x)} { {x^\top V^{-1} x} }   \\
      \leq &   \epsilon^2 \cdot   \frac{(V^{-1} x)^\top \cdot V \cdot (V^{-1} x)} { {x^\top V^{-1} x} }   \\
    =& \epsilon^2,
\end{align*}
and 
\[
    \langle \theta_1 - \theta_2 , x \rangle = \epsilon \| x \|_{V^{-1}} > \epsilon.
\]
Then we conclude that $x$ is $\epsilon$-independent of $\cS=\{x_1,\dots, x_n\}$.

\end{proof}

\begin{proof}[proof of Lemma~\ref{lem:submodularity}]
Denote $\Sigma_k = \sum_{i=1}^k x_i x_i^\top $.
The re-arrangement is done by sequentially choosing $x_i$ to be
\begin{align*}
    x_1 &= \argmax _{x_j: 1 \leq j \leq T} \log\det(I+\frac1\lambda (x_j x_j^\top) ),\\
    x_2 &= \argmax _{x_j: 2 \leq j \leq T} \log\det(I+\frac1\lambda (\Sigma_{1} + x_j x_j^\top) ),\\
        &\cdots \\
    x_i &= \argmax _{x_j: i\leq j \leq T} \log\det(I+\frac1\lambda (\Sigma_{i-1} + x_j x_j^\top) ), \\
    &\cdots 
\end{align*}
where ties are broken arbitrarily. By our procedure of choosing $x_i$, we have
\begin{equation}
\log\det(I+\frac1\lambda (\Sigma_{i-1} + x_i x_i^\top) ) \geq \log\det(I+\frac1\lambda (\Sigma_{i-1} + x_{i+1} x_{i+1}^\top) ).  \label{eq:greedy}
\end{equation}
Since $ \Sigma_i - \Sigma_{i-1}$  is positive semidefinite, we have
\begin{align}
 & \log\det(I+\frac1\lambda (\Sigma_{i-1} + x_{i+1} x_{i+1}^\top) )   - \log\det(I+\frac1\lambda \Sigma_{i-1} ) \notag \\
 =& \log (1 + \| x_{i+1} \|^2_ {(\lambda I +  \Sigma_{i-1} )^{-1}} ) \notag \\
  \geq& \log (1 + \| x_{i+1} \|^2 _ {(\lambda I +  \Sigma_{i} )^{-1}} ) \notag  \\
 =& \log\det(I+\frac1\lambda \Sigma_{i+1} ) - \log\det(I+\frac1\lambda \Sigma_{i} ) \label{eq:c3}.
\end{align}
Then by Equations~\eqref{eq:greedy} and~\eqref{eq:c3}, for $i=1,\dots,T-1$
\begin{align*}
    & \log\det(I+\frac1\lambda \Sigma_{i} ) - \log\det(I+\frac1\lambda \Sigma_{i-1} ) 
    \\
    =& \log\det(I+\frac1\lambda (\Sigma_{i-1} + x_i x_i^\top) ) - \log\det(I+\frac1\lambda \Sigma_{i-1} ) \\
    \geq&  \log\det(I+\frac1\lambda (\Sigma_{i-1} + x_{i+1} x_{i+1}^\top) ) - \log\det(I+\frac1\lambda \Sigma_{i-1} ) \\
    \geq& \log\det(I+\frac1\lambda \Sigma_{i+1} ) - \log\det(I+\frac1\lambda \Sigma_{i} ).
\end{align*}
This means
\[
    \Delta \gamma(\lambda, x_i|x_{1},\dots,x_{i-1}) \geq \Delta \gamma(\lambda, x_{i+1}|x_{1},\dots,x_{i}).
\]
\end{proof}

\end{document}